\documentclass[11pt,reqno]{amsart}
\usepackage[foot]{amsaddr}
\usepackage{amsthm,amssymb,amsfonts}
\usepackage{amsmath}
\usepackage{amscd} 
\usepackage{setup}
\usepackage{booktabs}

\usepackage[normalem]{ulem}
\usepackage[all]{xy}

\usepackage{thmtools,thm-restate}

\theoremstyle{definition}
\newtheorem{definition}{Definition}

\newtheorem{remark}{Remark}
\newtheorem{exmps*}{Examples}

\usepackage[utf8]{inputenc} % allow utf-8 input
\usepackage[T1]{fontenc}    % use 8-bit T1 fonts
\usepackage{hyperref}       % hyperlinks
\usepackage{url}            % simple URL typesetting
\usepackage{booktabs}       % professional-quality tables
\usepackage{nicefrac}       % compact symbols for 1/2, etc.
\usepackage{microtype}      % microtypography
\usepackage{xcolor}         % colors
\usepackage{natbib}

\newcommand{\R}{\mathbb{R}}
\newcommand{\D}{\mathcal{D}}
\newcommand{\Gen}{\mathrm{Gen}}

\newcommand{\E}{\mathbb{E}}

\newcommand{\Xor}{\mathrm{XOR}}

\renewcommand{\Pr}{\mathbb{P}}

\usepackage{tikz}
\usetikzlibrary{arrows.meta,positioning,calc}
\theoremstyle{plain}
\newtheorem{theorem}{Theorem}
\newtheorem{lemma}{Lemma}

\usepackage{xcolor}

\usepackage{algorithm}
\usepackage{algpseudocode}
\usepackage{graphicx}

\title[Genetic Algorithms with Optimization Guided Operators]{Genetic Algorithms with Optimization Guided Operators}

\author{Anna Brandenberger \and Ilan Doron-Arad \and Elchanan Mossel}
\address{Department of Mathematics, MIT}
\email{\{abrande,ilanda,elmos\}@mit.edu}

\begin{document}
	
	\maketitle

	\begin{abstract}
	Recent work in ML applies genetic algorithms at inference time to iteratively improve solutions to optimization problems. The basic mutation and recombination operators involved are qualitatively different from those studied classically. Mutations are no longer random; an ML algorithm mutates a solution with the goal of improving an objective. Similarly, recombination is not based on random collages of parent solutions. Instead, it is an ML optimization-based operator whose goal is to synthesize improved solutions from its inputs.
    Thus, these mutation and recombination operators are more likely to improve the objective, but their computational cost is much higher.
		
    We introduce a general model of genetic algorithms and formulate optimization
    in this model as a query complexity problem, using the language of reinforcement
    learning.
    We demonstrate three fundamental phenomena. 
    First, we show that {\em diversity} of the solution pool can be necessary: for parity learning, viewed in our framework, we show that with pool size $w$ and vectors of length $n$, the optimal query complexity is $\Theta(w+2^{n-w})$. We further show that this phenomenon persists under general memory constraints: $\Theta(n^2)$ bits of memory are necessary for efficient success.
    Second,  
    we show that generation, mutation, and recombination can all be simultaneously necessary to reach a nearly optimal solution. Finally, we give a phase transition for Gaussian distributions, showing that a positive {\em drift} of the operators yields exponential speedup.
	\end{abstract}

\section{Introduction}

Recent work has shown strong empirical success in using large language models (LLMs) as iterative search procedures for mathematical and algorithmic discovery. 
Important examples include AlphaEvolve~\citep{novikov2025alphaevolve}, FunSearch~\citep{romera2024mathematical},
Mind Evolution~\citep{lee2025evolving},
EvoPrompt \citep{guo2023connecting} 
and PromptBreeder~\citep{fernando2023promptbreeder}.
These processes are often described as {\em genetic algorithms}: the operations executed by
an LLM commonly include: (i) generating a fresh solution to the given problem; (ii) mutating
a given solution; and (iii) recombining two or more solutions to synthesize a new solution.
While described in the language of classical genetic algorithms (GAs), these LLM operations are
quite different from the random mutations and recombinations used by standard genetic
algorithms.

In classical GAs, \textit{mutation} and \textit{recombination} are stochastic, syntactic
operators over a fixed representation such as bit strings, real vectors,
permutations, or program trees~\citep{holland1975adaptation,goldberg1989genetic,koza1994genetic}. Both operators are uninformed, as they shuffle and perturb symbols without any understanding of what the genome means. Then, once the population is large after many generations, the results are sorted via selection pressure.
Thus, the division of labor in classical GAs is clean: recombination \emph{exploits}
existing variation by recombining what the population has already discovered,
while mutation \emph{explores} by injecting truly new material
\citep{eiben2015introduction}. 
 
In contrast, in LLM-driven evolution,
both operators are reimplemented as prompted LLM calls. 
The candidate is usually something the model can natively read and
write, such as a Python function, a prompt, a heuristic, a reward function, or a chain-of-thought template. 
Mutating a parent now produces a new candidate that is semantically informed: the LLM is prompted with instructions such as ``here is a solution that scores $X$; produce a modified version that
might score better'', sometimes augmented with the task description, recent
failure traces, or the parent's behavioral signature
\citep{lehman2023evolution,meyerson2024language}.
{Recombination} likewise goes beyond positional crossover, it now implements a \textit{semantic synthesis}: given multiple
parents, an LLM is asked to synthesize a candidate that combines their strengths or learns from their differences.
For example, FunSearch
\citep{romera2024mathematical} and AlphaEvolve \citep{novikov2025alphaevolve}
sample multiple high-scoring programs from an island and concatenate them in
the prompt as exemplars before asking for a new one, while Promptbreeder
\citep{fernando2023promptbreeder} and EvoPrompt \citep{guo2023connecting}
explicitly ask the model to ``cross over'' two prompts into a hybrid. Thus, mutation and recombination are better viewed as learned proposal
operators than as blind random perturbations.

This increased flexibility has a computational cost.
Classical operators are essentially free, so GAs
run millions of cheap evaluations over large populations~\citep{goldberg1989genetic}. LLM operators are expensive, so modern systems often shrink the population to dozens or low hundreds of candidates to accommodate a limited inference-time budget~\citep{romera2024mathematical}. 
This shift raises a fundamental question: under a fixed budget of expensive operator calls, how should one allocate queries among generation, mutation, and recombination to approach the performance of the best possible search strategy? Answering this question requires a framework that treats these operators on equal footing and makes their trade-offs explicit. 
	
\subsection*{Our Contribution}
	
Based on the above motivation, in this work we develop a theoretical
framework for analyzing ML-guided genetic algorithms. We formulate the generation--mutation--recombination (GMR) framework,
which models generation from scratch, mutation of one candidate, and
recombination of several candidates as stochastic operators, each with its own inference cost.
For a given inference budget, we study the maximum expected objective value attainable by an adaptive policy and how the budget should be allocated among generation, mutation, and recombination operators.
 This is aligned with the
empirical challenge of deciding how to spend a fixed budget of LLM calls on a hard
optimization problem. The general formulation of a genetic algorithm framework as black-box
optimization is not new~\citep{droste2006upper,lehre2012blackbox,doerr2011faster},
but the existing frameworks that have been studied are not fully suited to modern LLM use cases, which exhibit completely different behaviour.
We elaborate on related work in Section~\ref{sec:related_work}. In our framework, we prove three fundamental phenomena. 

\subsubsection*{Diversity matters.} 
Our main technical contribution is proving that diversity is necessary for genetic search problems. 
We use parity learning, a canonical model for learning an unknown vector $x\in\mathbb F_2^n$ from linear equations. 
Concretely, (candidate) solutions are linear equations \( a \cdot x = b\) for an unknown \( x \in \{0,1\}^n\): generation uniformly randomly samples new equations, and recombination XORs two existing equations coordinate-wise. 
The \textit{diversity} of a policy is measured by its live pool size, namely, the number of previously queried candidates that the policy keeps stored for use as arguments in future actions.
Informally, for constant success probability, we show that the optimal number of queries with live pool size $w$ is $\Theta(w+2^{n-w})$; thus, querying polynomially many times requires live pool size $w=n-O(\log n)$. We further show that this dependence on retained information is not an artifact of the live pool model: even allowing arbitrary compressed states, $\Theta(n^2)$ bits of memory are necessary for efficient success. To our knowledge, prior work establishes the necessity of diversity for specific evolutionary algorithms or separations between memory sizes one and two~\citep{storch2008population,storch2016memory,friedrich2009diversity,dang2018emergent,cerf2024diversity,sudholt2020diversity}; in contrast, our lower bounds apply to all adaptive policies.

\subsubsection*{All three operations are simultaneously necessary.} We then show that generation, mutation, and recombination can be simultaneously indispensable. 
We give a simple two-dimensional construction in which an
\(N\)-query policy can get close to the optimal \(N\)-query value, but any policy
that uses too few generations, too few mutations, or too few recombinations
cannot reach the corresponding approximation target. While prior work shows that recombination can be necessary or asymptotically
stronger than mutation only search~\citep{jansen2002crossover,jansen2005royalroad,
sutton2021crossover,doerr2011faster}, we are not aware of any result showing
that generation, mutation, and recombination are simultaneously necessary in
another genetic algorithm framework.

 \subsubsection*{Phase diagram by drift for Gaussian distributions.} 
    In a simple one dimensional Gaussian model, we show that optimization-guided mutation or recombination can improve exponentially faster than averaging. Gaussian search and biased operators have been studied before~\citep{beyer1993asymptotical,beyer1995benefits,ollivier2017information,chicano2024generalizing,eremeev2014recombination,zhang2005guided,kramer2008biased}, and our proof uses similar drift arguments. Nonetheless, to our knowledge, prior work does not identify this transition from averaging to positive bias.

\subsection*{Organization}
Section~\ref{sec:model} defines the GMR framework and its query-complexity objective, and Section~\ref{sec:related_work} discusses related work. We then present our three main results, ordered by their simplicity: Section~\ref{sec:one-dimensional-gaussian} establishes the Gaussian phase transition, Section~\ref{sec:allOps} shows that generation, mutation, and recombination can all be simultaneously necessary, and Section~\ref{sec:parity} proves the necessity of diversity through a tight query-complexity versus live-pool-size tradeoff. Sections~\ref{app:one-dimensional-gaussian} and~\ref{app:allOps} give the full proofs of the first two results, and Section~\ref{sec:discussion} concludes.

\section{A Formal Model of Genetic Algorithms}
\label{sec:model}
    
We give the formal definition of the model below. At a high level, generation
captures producing a fresh candidate from the model; mutation describes
refinements of one existing solution; recombination depicts
synthesizing a new solution from two or more existing ones. As our main
objective, we study the framework through the lens of query complexity.
Concretely, for a given budget \(n\), we ask for the maximum expected objective value attainable by a policy of total cost at most \(n\). This viewpoint is especially natural in our setting,
where a policy interacts with a hidden optimization problem and may observe
only limited feedback about the candidates it produces. The objective \(U\) plays the role of the reward function: it assigns a scalar
value to each queried candidate, and the policy aims to maximize the best value
seen over its query budget.

\begin{definition}
	A \emph{generation--mutation--recombination} (GMR) framework is a tuple
	\[
	\mathfrak G=(X,U, \mathcal{D},c,\mathsf{Obs}),
	\]
	where \(X\) is a space of candidate solutions and \(U:X\to\R\) is a hidden
	objective; moreover, $\mathcal{D} = \bigcup_{k \in K} \mathcal D_k$ where \(K \subseteq \mathbb{Z}_{\geq 0}\) with $0 \in K$, and, for each \(k \in K\),
	\(\mathcal D_k\) is a collection of \(k\)-input distribution rules. Each
	\(D\in\mathcal D_k\) assigns to every \((x_1,\ldots,x_k)\in X^k\) a
	distribution \(D(x_1,\ldots,x_k)\) over \(X\). 
    Finally, $c:\mathcal{D} \rightarrow \mathbb{N}_{>0}$ gives each $D \in \mathcal{D}$ an inference cost $c(D)$. 
	
	The case \(k=0\) corresponds to {\em generation} from scratch. The case
	\(k=1\) corresponds to {\em mutation} or refinement of one existing
	candidate. The cases \(k\ge 2\) correspond to {\em recombination} or
	synthesis from several existing candidates. All expectations below are assumed
	to be well-defined.

    \medskip
    \noindent {\bf Feedback model:} The map \(\mathsf{Obs}\) specifies the feedback observed
	by the policy. After a candidate \(z_t\in X\) is produced, the policy observes
	\(
	o_t=\mathsf{Obs}(z_t;z_1,\ldots,z_{t-1}).
	\)
	The policy need not observe \(z_t\) itself; feedback may include scores,
	labels, or similarities to previously generated candidates. For example, full
	feedback is the case \(o_t=(z_t,U(z_t))\), while objective-only feedback is the
	case \(o_t=U(z_t)\). In the latter case, the policy may still refer to
	previously queried candidates by index when applying mutation or recombination.
	
	\medskip
    \noindent {\bf Policies:} Let \(\Pi_{\mathfrak G}^{\mathrm{obs}}\) denote the set of
possibly randomized policies that use only the available operators and the
feedback history generated by \(\mathsf{Obs}\).
Policies may invoke an available operator but do not otherwise access its
underlying distribution or the hidden objective \(U\), except through the
observed feedback. A policy
	\(\pi\in\Pi_{\mathfrak G}^{\mathrm{obs}}\), at each round \(t\in\mathbb N\),
	after observing the feedback history \(o_1,\ldots,o_{t-1}\), chooses an
	integer \(k\in K\) and a distribution rule \(D\in\mathcal D_k\). If \(k=0\), it
	queries \(z_t\sim D\). If \(k\ge 1\), it chooses previous indices
	\(i_1,\ldots,i_k<t\), not necessarily distinct, and queries
	\(
	z_t\sim D(z_{i_1},\ldots,z_{i_k}).
	\)
	At round \(t=1\), only generation from some \(D\in\mathcal D_0\) is feasible. For a policy $\pi$ that performed operations $D_1,\ldots,D_T \in \mathcal{D}$, let $\textsf{cost}(\pi,T) = \sum_{1 \le t \le T} c(D_t)$ be its cost.   
	
	\medskip
    \noindent {\bf Query Complexity:}
    For a policy $\pi$ 
    and $T \in \mathbb{N}$, let $z_1,\ldots, z_T$ be the first $T$ solutions created by $\pi$. Let the expected value of $\pi$ after $T$ iterations be 
	\[
	V_T^\pi(\mathfrak G):=\E \left[\max_{t\in[T]} U(z_t)\right]
	\]
	Define the optimal value obtainable by a policy under budget $n$ as
	\[
	V_n^*(\mathfrak G):=
	\sup_{\substack{
\pi\in\Pi_{\mathfrak G}^{\mathrm{obs}},\ T\in\mathbb N\\
\mathsf{cost}(\pi,T)\le n
}}V_T^\pi(\mathfrak G).
	\]
	When the framework is fixed, we suppress it from the notation.
\end{definition}

\begin{remark}
	In our model, the value of the objective at a point \(x\in X\) can only be observed
    if \(x\) is produced by one of the allowed distributions. In particular,
	the policy cannot query an arbitrary point directly. Moreover, depending on
	\(\mathsf{Obs}\), the policy may not observe the internal candidate \(x\) at
	all. Thus, the problem is not purely one of optimizing an explicitly accessible
	search space, but also of generating candidate points and learning from the
	revealed feedback. This is natural in our setting, where points in \(X\) represent
	candidate solutions or algorithms, and arbitrary points are not
	assumed to be directly accessible.
\end{remark}

\noindent
\textbf{Examples.}
In practice, the distributions in \(\mathcal D_k\) may be interpreted
as those induced by fixed prompts or model calls for generation, mutation, and
recombination. A generation step samples \(z\sim D\) for some
\(D\in\mathcal D_0\), for example by asking the model to solve the task from
scratch. A mutation step samples \(z\sim D(x)\) for some \(D\in\mathcal D_1\) and input $x \in X$,
which may correspond to a local revision such as fixing a bug in a program,
improving one step in a proof, or refining part of a plan while keeping the rest
unchanged. A recombination step samples \(z\sim D(x_1,\ldots,x_k)\) for some
\(k\ge 2\) and \(D\in\mathcal D_k\), which may correspond to synthesizing several
partial solutions that contain different useful ingredients. 
For example, recombination may merge complementary components from different candidates, such as code, tests, or proof ideas.
Specific instantiations of \(\D_k\) can yield many familiar
stochastic search and model training procedures, such as stochastic gradient
descent, bandits, evolution search and more; see details in
Appendix~\ref{sec:examples}.

	\section{Related Work}
	\label{sec:related_work}

 Classical genetic algorithms are commonly viewed as stochastic search over fitness landscapes using unbiased mutation and recombination operators~\citep{wright1932roles,kauffman1993origins,holland1975adaptation,goldberg1989genetic}. This is no longer the case in ML-guided genetic search, as the operators are much more expensive and intricate. Below, we give the closest works in terms of both model and results, and then briefly discuss practical works.   

    \subsection*{Related models}
     In the most general models of black-box optimization, the next query can depend on all information obtained so far, without being created via an operation applied to specific previous solutions~\citep{droste2006upper}. More restricted models are closer to ours, such as models generating a new solution from one previously queried solution~\citep{lehre2012blackbox}, or from \(k\) previously queried solutions~\citep{doerr2011faster}. However, in these works, the operators are assumed to be unbiased, meaning that they cannot favor particular bit values or coordinates.
Gray-box models assume that some structure of the objective function is known, such as an explicit decomposition into local terms. Their operators use this structure to modify or combine parts of a solution that affect particular terms of the objective~\citep{whitley2016graybox,baguley2022analysis,chicano2024generalizing}. This is different from ML genetic algorithms, which generally do not admit such known decompositions.  

\subsection*{Operator costs in genetic algorithms.}
Prior works in the theory of genetic algorithms consider several different notions of cost. \citep{tuson1998adapting} define cost as the computational effort required to evaluate a new solution, while \citep{lehre2025runtime} study settings in which fitness evaluations have different costs and note that evaluation often accounts for most of the computation. This differs from ML settings, where creating a solution may be substantially more expensive than evaluating it. \citet{kneissl2026cost} instead measure the cost of randomness, namely, the number of random bits used by mutation and crossover. A closer notion to our notion of costs appears in \citep{maturana2008compass}, who consider the computational costs of different operators and use them to calibrate a specific heuristic so that slower operators are applied less often; however, they do not study the optimal allocation of a cost budget among the operators.

     \subsection*{Related Results}The works of Beyer consider the objective $\lVert x\rVert_2^2$ and are related to our Gaussian result because they analyze Gaussian mutation and parent averaging~\citep{beyer1993asymptotical,beyer1995benefits}. The work of \citet{ollivier2017information} is also related, as it adapts the parameters of a Gaussian search distribution based on previous solutions. However, the focus of these works is different from ours as they analyze particular evolutionary strategies whereas we consider general policies and give a phase transition based on the drift of the operators. 
     
     Related to our second result, that shows generation, mutation, and recombination are all necessary, there is prior work showing that crossover is necessary or asymptotically stronger for particular problems or algorithms, such as royal-road functions and the closest-string problem~\citep{jansen2005royalroad,sutton2021crossover}. We did not find a result showing that generation, mutation, and recombination are all necessary simultaneously.
     
     Related to our diversity result, \citet{storch2016memory} gives an exponential separation between memory sizes one and two, while \citet{storch2008population} studies other population sizes for specific evolutionary search methods. The work of \citet{friedrich2009diversity} studies a bimodal setting and shows that diversity is needed to find both optima efficiently, while \citet{cerf2024diversity} show that diversity is necessary for crossover to give a substantial speedup for a particular genetic algorithm. In sharp contrast, our diversity lower bound applies to all adaptive policies and live pool sizes, rather than to a particular algorithm or for fixed memory sizes.

	\subsection*{Machine Learning Evolutionary Search} 
    Evolutionary methods have long been used within machine learning
\citep{stanley2002evolving,salimans2017evolution,such2017deep,real2019regularized,jaderberg2017population,qiu2026evolution}. 
    Recent work uses LLM-based genetic algorithms for evolve programs and algorithms, such as FunSearch and AlphaEvolve~\citep{romera2024mathematical,novikov2025alphaevolve,nagda2026reinforced}; other use cases are language model evolution~\citep{lehman2023evolution,meyerson2024language}, prompt optimization and model merging~\citep{guo2023connecting,zhao2023genetic,akiba2025evolutionary}, and reasoning frameworks such as Lyria, FlowBoost, and PatternBoost~\citep{tang2025lyria,berczi2026flow,patternboost2024}.
	In these modern architectures, mutations and recombinations are qualitatively different from classical genetics, and
    our framework provides a theoretical query-complexity language for analyzing these new genetic operators. 
Our framework can also be viewed through a reinforcement-learning lens: for each fixed objective it induces a finite horizon Markov decision process (MDP), while from the learner's perspective it can be viewed as a partially observable MDP or, under a prior over objectives, as Bayesian reinforcement learning~\citep{puterman1994markov,sutton2018reinforcement,kaelbling1998planning,ghavamzadeh2015bayesian}.

\section{Beyond Averaging Operators: A One-Dimensional Gaussian Example}
\label{sec:one-dimensional-gaussian}

We first give a simple one-dimensional instantiation of the GMR framework illustrating how optimization-guided operators can be much stronger than averaging-type mutation and recombination operators. Let \(\mathfrak G=(X,U,\mathcal{D},c,\mathsf{Obs})\) be defined as follows. Let \(X=\mathbb R\), let \(U(x)=x\), and use full feedback, i.e., \(o_t=(z_t,U(z_t))\). Fix \(\mu,\lambda_1,\lambda_2\ge0\) and \(\sigma_M,\sigma_R>0\). Generation samples an independent standard Gaussian \(N(0,1)\) random variable, mutation maps \(x \mapsto \mu x+\sigma_M Z\), and recombination maps \((x,y) \mapsto \lambda_1x+\lambda_2y+\sigma_R Z\), where the \(Z\sim N(0,1)\) variables are independent. Thus, \( \mathcal{D} = \mathcal D_0 \cup \mathcal D_1 \cup\mathcal D_2\), and consider any cost function $c:\mathcal{D} \rightarrow \mathbb{N}$ with strictly positive image.  

\begin{restatable}{theorem}{onedimgaussian}
\label{thm:one-dimensional-gaussian}

Let
$
\alpha:=\min\left\{1,\frac{1-\mu}{\sigma_M},\frac{1-\lambda_1-\lambda_2}{\sigma_R}\right\}.
$
Then:
\begin{itemize}
\item If $\alpha>0$, then $V_n^*(\mathfrak G)=(1+o(1))\alpha^{-1}\sqrt{2\log n}$.
\item If $\alpha=0$, then $V_n^*(\mathfrak G) =  \Theta(n)$.
\item If $\alpha<0$, then $V_n^*(\mathfrak G) =  \exp(\Theta(n))$.
\end{itemize}
\end{restatable}

The contractive case, where \(\mu<1\) and \(\lambda_1+\lambda_2<1\), shows that contractive averaging mutation or recombination still gives only the usual Gaussian extreme-value rate, up to the sharp constant. In contrast, once mutation or recombination preserves or amplifies the parent value, repeated guided improvement can be linear or exponential. This demonstrates a basic distinction between classical GAs, where recombination is typically an averaging or random-splicing operation, compared to ML GAs with optimization-guided operators.

\medskip \noindent
\textbf{Proof sketch.} For the upper bound, look at the first time the best value exceeds a value \(u \in \mathbb R\). If all previous values are at most \(u\), then generation, mutation, and recombination can exceed \(u\) only through Gaussian fluctuations of sizes \(u\), \((1-\mu)u/\sigma_M\), and \((1-\lambda_1-\lambda_2)u/\sigma_R\), respectively. A union bound over \(n\) queries gives the upper tail. The lower bound repeatedly applies the operator attaining the minimum in \(\alpha\) and climbs through logarithmically many levels. The non-contractive lower bounds follow by repeatedly applying the non-contractive operator. Details are in Section~\ref{app:one-dimensional-gaussian}.

This example captures one distinction between classical and ML-guided genetic algorithms: directed mutation or recombination can improve solution quality beyond what blind averaging can achieve. At the same time, it is too simple to explain why generation, mutation, and recombination may all be needed simultaneously, or why a diverse solution pool is necessary. These phenomena require the richer constructions that we develop in the following sections.

\section{All Three Operators Are Necessary}
\label{sec:allOps}
Next, we give a simple setting where all three operators, namely generation, mutation, and recombination, are all necessary to obtain a constant factor approximation from the optimal solution.

\begin{restatable}[All three operators are necessary]{theorem}{threeOps}
\label{thm:all-three-short}
There is a universal constant $\gamma>0$ such that for all $G,R,M\in\mathbb N_{>0}$ and $c_G,c_R,c_M\in\mathbb N_{>0}$, there is a GMR framework $\mathfrak G=(X,U,\mathcal D,c,\mathsf{Obs})$ with $U:X\to[0,3]$, full feedback, and one generation, recombination, and mutation operator of costs $c_G,c_R,c_M$, respectively, such that,
for $B:=10Gc_G+Rc_R+Mc_M$, the following holds: Any policy $\pi$ and $T\in\mathbb N$ satisfying
$V_T^\pi(\mathfrak G)\ge \frac9{10}V_B^*(\mathfrak G)$ must, with probability at least $1/2$, simultaneously apply at least $\gamma G$ generations, $\gamma M$ mutations, and $\gamma R$ recombinations.
\end{restatable}

In particular, any policy achieving at least $9/10$ of the optimal cost $B$ value must, with constant probability, use $\Omega(G)$ generations, $\Omega(R)$ recombinations, and $\Omega(M)$ mutations.

	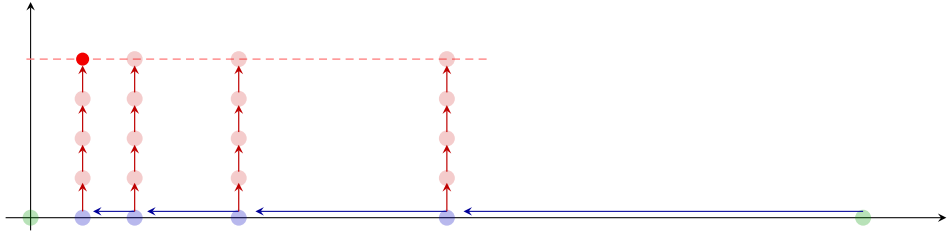
\begin{figure}[h]
		\centering
		\begin{tikzpicture}[x=11cm,y=4.2cm,>=stealth]
			
			% axes
			\draw[->, thin] (-0.03,0) -- (1.10,0);
			\draw[->, thin] (0,-0.04) -- (0,0.68);
			
			% generation points
			\fill[green!60!black, opacity=0.28] (0,0) circle (3pt);
			\fill[green!60!black, opacity=0.28] (1,0) circle (3pt);
			
			% recombination points on the x-axis
			\foreach \x in {0.5,0.25,0.125,0.0625}{
				\fill[blue!75!black, opacity=0.28] (\x,0) circle (3pt);
			}
			
			% simple recombination chain
			\draw[->, blue!60!black, line width=0.4pt] (1,0.02) -- (0.52,0.02);
			\draw[->, blue!60!black, line width=0.4pt] (0.5,0.02) -- (0.27,0.02);
			\draw[->, blue!60!black, line width=0.4pt] (0.25,0.02) -- (0.14,0.02);
			\draw[->, blue!60!black, line width=0.4pt] (0.125,0.02) -- (0.075,0.02);
			
			% target height
			\draw[densely dashed, red!50] (-0.005,0.5) -- (0.55,0.5);
			
			% mutation chains from the x-axis upward
			\foreach \x in {0.5,0.25,0.125,0.0625}{
				\draw[->, red!75!black, line width=0.45pt] (\x,0.02) -- (\x,0.11);
				\draw[->, red!75!black, line width=0.45pt] (\x,0.145) -- (\x,0.23);
				\draw[->, red!75!black, line width=0.45pt] (\x,0.27) -- (\x,0.355);
				\draw[->, red!75!black, line width=0.45pt] (\x,0.395) -- (\x,0.48);
			}
			
			% candidate peak columns (faint)
			\foreach \y in {0.125,0.25,0.375,0.5}{
				\fill[red!80!black, opacity=0.20] (0.5,\y) circle (3pt);
				\fill[red!80!black, opacity=0.20] (0.25,\y) circle (3pt);
				\fill[red!80!black, opacity=0.20] (0.125,\y) circle (3pt);
			}
			
			% true peak column: lower points faint, top point bold
			\foreach \y in {0.125,0.25,0.375}{
				\fill[red!80!black, opacity=0.20] (0.0625,\y) circle (3pt);
			}
			\fill[red!95!black] (0.0625,0.5) circle (2.45pt);
			
		\end{tikzpicture}
		\caption{Illustration of the two-dimensional construction. Generation provides the (green) points $(0,0)$ and $(1,0)$; recombination (blue) is midpoint recombination, allowing one to scan the $x$-axis; mutation (red) lifts these locations toward height $1/2$. The bold red point indicates the unique peak of the hidden objective, while the other faint red points illustrate intermediate dyadic locations.}
		\label{fig:three-ops-2d}
	\end{figure}

\noindent
\textbf{Proof sketch.}
Generation produces the point $(0,0)$ with high probability and the point $(1,0)$ with low probability. The recombination rule is midpoint, so once both generated points are available, one can query \(a_1=(1/2,0)\), and then repeatedly recombining with \((0,0)\) gives
\(a_j=(2^{-j},0)\) for \(j=1,\ldots,R\). The objective is
\(U(x,y)=3\cdot 2^{-\Lambda |x-2^{-R}|}\cdot \max\{0,1-|y-1/2|\}\)
for \(\Lambda=2^{R+10}\), and the feedback is full feedback:
\(\mathsf{Obs}(z;z_1,\ldots,z_{t-1})=(z,U(z))\).
Thus, \(U\) has its unique peak at \(b=(2^{-R},1/2)\), and its value decays exponentially in the horizontal error and linearly in the vertical error. Hence, generation is needed to see \((1,0)\), recombination is needed to reach the required horizontal location \(2^{-R}\), and mutation is needed to climb from height \(0\) to height \(1/2\). The upper-bound policy performs these three phases in sequence, using \(10G\) generations, \(R\) recombinations, and \(M\) mutations, for total cost \(B=10Gc_G+Rc_R+Mc_M\). For the lower bound, too few generations miss \((1,0)\) with high probability; too few recombinations keep all queried \(x\)-coordinates away from the required location \(2^{-R}\), so their objective value is negligible; and too few mutations keep all queried points far below height \(1/2\). Combining these three bounds shows that any policy achieving at least \(9/10\) of \(V_B^*(\mathfrak G)\) must, with probability at least \(1/2\), simultaneously use \(\Omega(G)\) generations, \(\Omega(R)\) recombinations, and \(\Omega(M)\) mutations. See Figure~\ref{fig:three-ops-2d} for an illustration. The full proof appears in Section~\ref{app:allOps}.

\section{A Model where Diversity is Necessary: Parity Learning}
\label{sec:parity}

In this section, we prove that \textit{solution diversity} can be
necessary. That is, this problem requires keeping a large and sufficiently independent
pool of intermediate candidates alive to be used as arguments for future operators. We study parity learning, where each
generated candidate is a random linear equation over $\mathbb{F}^n_2$ satisfied by an unknown string
\(x\in\{0,1\}^n\), and recombination corresponds to XORing candidates. 
The objective of the policy is to recover the unknown \(x\).
This model fundamentally requires diversity of the solution pool: the recombination operation cannot increase the span of the set of existing equations, so any policy cannot recover the hidden target \(x\) without generating and maintaining a sufficiently large pool of linearly independent candidates. Next, we formally define the diversity of a policy for a given framework and then define parity learning as a GMR framework.

\subsection*{Diversity}
We define the \textit{diversity} of a policy by the size of its \textit{live pool}. We say that a policy \(\pi\) has \emph{live pool size \(w\)} up to time \(T\) if, after every round \(t<T\), its entire persistent state is an ordered list \(P_t\) of at most \(w\) previously queried candidates. Conditional on \(P_t\) and \(t\), the distribution of all future actions is independent of the earlier history. In particular, every candidate used as an argument to a future operator must belong to the current live pool.

\subsection*{Parity learning.} Formally, we formulate parity learning as a family of GMR instances indexed by the hidden
target \(x\in\{0,1\}^n\). Let $\oplus$ denote the coordinate wise XOR operator. Fix \(x\in\{0,1\}^n\), and define
$$ \mathfrak G_x^{\mathrm{par}} = (X,U_x,\mathcal D,c,\mathsf{Obs}) $$
as follows. 
Let $X:=\{0,1\}^n\times\{0,1\}$ and $\mathcal D:=\{\Gen_x,\Xor\}$ defined by
\begin{itemize}
    \item $\Gen_x()$ returns $(a,b)$, where $a\sim\mathrm{Unif}(\{0,1\}^n)$ and $b=a\cdot x \mod 2$.
    \item $\Xor((a,b),(a',b')):=(a\oplus a',\,b\oplus b')$ for every $(a,b),(a',b')\in X$.
\end{itemize}
This preserves validity, since
$(a\oplus a')\cdot x = (a\cdot x)\oplus(a'\cdot x)=b\oplus b'$. We use $c(\Gen_x) = 1$ and $c(\Xor) = 1$ for simplicity; the same bounds will hold for all fixed costs. Finally, the utility function satisfies $U_x(a,b):=\mathbf 1\{a=x\}$, i.e., the only point with positive utility is $(x,x\cdot x)$. Throughout this section, \(x\sim\mathrm{Unif}(\{0,1\}^n)\); policies are chosen independently of \(x\), and \(V_T^\pi\) also averages over the choice of \(x\).

\medskip \noindent
\textbf{Feedback model.}
We use full feedback, $\mathsf{Obs}(z;z_1,\ldots,z_{t-1})=(z,U_x(z))$.

\medskip 

The main result of this section is a tight bound on the query complexity needed for a policy with given live pool size to solve parity learning with constant probability. We show that $\Theta(w+2^{n-w})$ queries are exactly what is needed.

\begin{theorem}[Diversity is necessary]
\label{thm:parity-tradeoff}
There are \(c,C>0\) such that for every \(n\ge2\) and \(1\le w\le n\) the following holds.
\begin{enumerate}
\item There is a policy \(\pi\) with live pool of size $w$ and
\(T \le C(w+2^{n-w})\) such that $$V_T^\pi(\mathfrak G_x^{\mathrm{par}}) \ge2/3.$$
\item For every 
policy $\pi$ with live pool size at most $w$, if
$V_T^\pi(\mathfrak G_x^{\mathrm{par}})\ge2/3$, then
$$
T\ge c(w+2^{n-w}).
$$
\end{enumerate}
\end{theorem}

In the above result, the live pool size prevents policies from using extra information that is not in the live pool. We strengthen the lower bound by considering a general memory restriction requirement on the policies but allow unbounded live pool size. 

\subsection*{Bounded memory policies.}
For \(M\in\mathbb N\), we say that a policy \(\pi\) has \emph{memory \(M\)} up to time \(T\) if, after every round \(t\le T\), its entire persistent state can be represented by a random variable
$
S_t\in\{0,1\}^M.
$
Conditional on \(S_t\) and \(t\), the distribution of all future actions is independent of the earlier history. In particular, any information about previously queried candidates that the policy wishes to use later, including candidates used as inputs to future operators, must be encoded in \(S_t\). Under this restriction, previously queried candidates are not otherwise
accessible by index: a candidate may be used as an input to a future operator
only if its value is recoverable from \(S_t\). Fresh independent randomness is not counted toward the memory bound.

\begin{restatable}[Memory is necessary]{theorem}{ParityMemoryBarrier}
\label{thm:parity-memory}
There exist \(c,C,\alpha>0\) such that for every \(n\ge2\): 
\begin{enumerate}
\item There is a policy \(\pi\) using at most \(Cn^2\) bits of memory and
\(T\le Cn\) queries such that
$$   V_T^\pi(\mathfrak G_x^{\mathrm{par}})\ge 2/3.
   $$

\item Every policy \(\pi\) using at most \(cn^2\) bits of memory satisfies for every \(T\le2^{\alpha n}\)
\[
V_T^\pi(\mathfrak G_x^{\mathrm{par}})
\le O(2^{-\alpha n})
\]
\end{enumerate}
\end{restatable}

The proof of this theorem uses a reduction to streaming algorithms for parity learning, based on the lower bound of~\citet{raz2018fast}. It remains open to obtain an analogue of Theorem~\ref{thm:parity-tradeoff} for general memory bounded policies, namely, a tight query complexity tradeoff for every memory bound \(M\in[n,n^2]\) in the setting of Theorem~\ref{thm:parity-memory}.
In the rest of this section, we prove Theorem~\ref{thm:parity-tradeoff} and then Theorem~\ref{thm:parity-memory}. 

\begin{remark}[Islands]
The same framework gives an island interpretation, relevant to practical evolutionary systems such as FunSearch and AlphaEvolve~\citep{romera2024mathematical,novikov2025alphaevolve}. Partition the coordinates into two blocks and let generation sample uniformly from either corresponding coordinate subspace. Since recombination cannot enlarge the span, constant success probability requires retaining a linear number of independent samples from each island. The construction extends identically to any fixed number of islands.
\end{remark}

\subsection{Proof of Theorem~\ref{thm:parity-tradeoff}.} We first prove the lower bound. The key observation is that recombination never enlarges the span of the live pool, so the target can become reachable only when a fresh generation vector enlarges this span to contain \(x\). Conditioning on the history before each generation fixes the current span, while the new vector remains uniform, which lets us bound the probability that this happens at each generation step and sum these probabilities over all generations.

\begin{lemma}
    \label{lem:lower}
    There exists \(c > 0\) such that for every \(n\ge2\) and \(1\le w\le n\), $T \in \mathbb{N}$, and every policy $\pi$ with live pool size at most $w$, if
$V_T^\pi(\mathfrak G_x^{\mathrm{par}})\ge2/3$, then
$
T\ge c(w+2^{n-w}).
$
\end{lemma}

\begin{proof}
We track only the first coordinate of each candidate. Recombination cannot enlarge the span of the live pool, since \(a\oplus a'\in\operatorname{span}\{a,a'\}\); thus only generation can introduce a new direction.

We number the generation queries in order, and sample the corresponding independent
\(A_1,A_2,\ldots\sim\operatorname{Unif}(\mathbb F_2^n)\), independently of \(x\) and the internal randomness \(R\) of \(\pi\). Let
\[
\mathcal H_{j-1}:=\sigma(x,R,A_1,\ldots,A_{j-1}).
\]
Let \(E_j\) be the event that the \(j\)-th generation occurs within the first \(T\) rounds. On \(E_j\), let \(S_j\) be the span of the first coordinates in the live pool immediately before this generation, and set \(S_j:=\{0\}\) otherwise. By the definition of the live pool, \(E_j\) and \(S_j\) are \(\mathcal H_{j-1}\)-measurable, while \(A_j\) is uniform and independent of \(\mathcal H_{j-1}\). Moreover,
\[
\dim S_j\le \min\{w,j-1\},
\]
since the pool contains at most \(w\) candidates and every candidate before the \(j\)-th generation lies in \(\operatorname{span}(A_1,\ldots,A_{j-1})\).
Define the event that $x$ is first included in the span of the live pool at the $j$-th generation,
\[
B_j:=E_j\cap\{x\notin S_j\}\cap\{x\in\operatorname{span}(S_j,A_j)\}.
\]
On the event \(E_j\cap\{x\notin S_j\}\), we have \(x\in\operatorname{span}(S_j,A_j)\) if and only if \(A_j\in x+S_j\). Thus,
\[
\Pr(B_j\mid\mathcal H_{j-1})
=
\mathbf 1_{E_j\cap\{x\notin S_j\}}
\frac{|S_j|}{2^n}
\le
2^{\min\{w,j-1\}-n},
\]
and so taking expectations gives \(\Pr(B_j)\le2^{\min\{w,j-1\}-n}\).

If \(x\ne0\) and the policy ever queries the target, then \(x\) must enter the span of the live pool for the first time at some generation query, since recombination and deletion cannot enlarge this span. Thus,
$
\left\{\max_{t\in[T]}U_x(z_t)=1\right\}
\subseteq
\{x=0\}\cup\bigcup_{j=1}^T B_j
$,
and 
\[
V_T^\pi
\le
2^{-n}+\sum_{j=1}^T2^{\min\{w,j-1\}-n}.
\]

In particular, $V_T^\pi\le2^{T-n}$ and
$V_T^\pi\le2^{-n}+T2^{w-n}$.
If \(V_T^\pi\ge2/3\), the first bound gives \(T\ge n\ge w\), while the second
gives $T2^{w-n}\ge\frac23-\frac14=\frac5{12}$,
so \(T\ge\frac5{12}2^{n-w}\). Together this proves the lemma. 
\end{proof}

We now prove the matching upper bound. The idea is to store \(w-O(1)\) independent equations that determine all but \(n-w+O(1)\) degrees of freedom of \(x\). Each subsequent generation then gives a candidate guess for \(x\) that is correct with probability \(\Omega(2^{-(n-w)})\); before spending \(O(w)\) recombinations to construct a guess, we use fresh generated equations to cheaply reject almost all incorrect guesses.

\begin{lemma}
\label{lem:upper}
There is \(C>0\) such that for every \(n\ge2\) and \(1\le w\le n\) there is a policy \(\pi\) with live pool of size $w$ and
\(T \le C(w+2^{n-w})\) such that $V_T^\pi(\mathfrak G_x^{\mathrm{par}}) \ge2/3$. 
\end{lemma}

\begin{proof}
We first construct a policy with expected query complexity \(O(w+2^{n-w})\), then truncate it. The policy is defined as follows, with $\rho > 0$ to be determined later. \medskip  

\noindent\textbf{Policy $\pi$.}
For \(w\ge5\): let $r = w-4$ and $r \le s \le r+1$.
\begin{enumerate}
\item Build a basis \(u_1,\ldots,u_s\) of a nondegenerate subspace \( S := \mathrm{span}(u_1, \dots, u_s)\), corresponding to valid equations \( (u_i, \beta_i)_{i=1}^s\) where $\beta_i = u_i \cdot x$.
\item Compute the unique \(q\in S:=\operatorname{span}(u_1,\ldots,u_s)\) satisfying \(u_i\cdot q=\beta_i\).
\item Generate a nonzero seed \(A\), write \(A = s_A+h_A \) with \(s_A \in S\) and \(h_A \in S^\perp\), and set \(\widehat x:=q+h_A\).
\item Screen \(\widehat x\) with fresh parity equations; reject on failure, else accept with probability \(\rho\).
\item If accepted, XOR \(A\) with the basis vectors $u_i$ to produce \(\widehat x\), then repeat.
\end{enumerate}
Stop whenever the target is queried.
\medskip 

If \(w\le4\), repeatedly generate fresh candidates. Each query hits the target with probability \(2^{-n}\), so the expected number of queries is \(2^n\le16\cdot2^{n-w}\). Truncating after \(3\cdot2^n\) queries gives success probability at least \(2/3\). Hence we can assume \(w\ge5\), and set \(r:=w-4\).
\smallskip 

\textit{Step 1.} We build $S$ incrementally from $S = \{0\}$, stopping as soon as $\dim S \ge r$ and ensuring that $S$ is nondegenerate,
i.e., \(S\cap S^\perp=\{0\}\).
Write a fresh uniform vector uniquely as \(a=s_a+h_a\) with \(s_a\in S\) and \(h_a\in H:=S^\perp\), where \(h_a\) is uniform in \(H\).

The map \(h\mapsto h\cdot h\) on \(H\) is linear. If it is nonzero, generate until \(h_a\cdot h_a=1\), which takes two generations in expectation. Since \(h_a\perp S\) and \(h_a\cdot h_a=1\), adjoining \(a\) increases the dimension by one and preserves nondegeneracy.

Otherwise \(h\cdot h=0\) for every \(h\in H\). The inner product restricted to \(H\) is still nondegenerate, since \(H\cap H^\perp=S^\perp\cap S=\{0\}\). Generate \(a_1=s_1+h_1\) with \(h_1\ne0\), and then \(a_2=s_2+h_2\) with \(h_1\cdot h_2=1\). Both take constant expected generations. Since the Gram matrix of \(h_1,h_2\) is
$\left(
\begin{smallmatrix}
0&1\\
1&0
\end{smallmatrix} \right),
$
adjoining \(a_1,a_2\) increases the dimension by two and preserves nondegeneracy. 

Thus, since each round adds one or two dimensions with positive probability, in expectation, after \(O(w)\) generations we obtain \(S\) with \(w-4\le s\le w-3\).

\smallskip 
\textit{Step 2.}
Since \(S\) is nondegenerate, there is a unique \(q\in S\) satisfying \(u_i\cdot q=\beta_i\) for every \(i\). Indeed, the map \(q\mapsto(u_1\cdot q,\ldots,u_s\cdot q)\) is bijective. Since \(\beta_i=u_i\cdot x\), we have \(x+q\in S^\perp\).

\smallskip 
\textit{Step 3.}
Set 
\(d:=n-s\le n-w+4\) 
and let \(h_*:=x+q\in H\). Then \(\widehat x=x\) exactly when \(h_A=h_*\).
Condition on $x$ and on the outcome of Step 1. 

Each fresh seed is independent and uniform, so conditioned on \(A\ne0\), the successful seeds are the nonzero elements of $S+h_*$, of which there are $2^s$ if \(h_*\ne0\) and $2^s - 1$ if $h_* = 0$, so in both cases
\[
\Pr(\widehat x=x) \ge \frac{2^s-1}{2^n-1}\ge2^{-d-1} := \eta,
\]
where 
$
\eta^{-1}\le2^{n-w+5}.
$
So obtaining a nonzero seed takes fewer than two generations in expectation.

\smallskip
\textit{Steps 4-5.} To avoid spending \(O(w)\) recombinations on every seed, we screen its guess. Given \(\widehat x\), generate a fresh equation \((v,\alpha)\), where \(v\cdot x = \alpha\), and test whether \(v\cdot\widehat x=\alpha\). A correct guess always passes, while a wrong guess passes independently with probability \(1/2\).

Set \(\rho:=\eta/8\). At each screening step, accept the guess with probability \(\rho\); otherwise perform one test, rejecting the seed if it fails and continuing if it passes. For a wrong guess, if \(f\) is its probability of being accepted before rejection, then
$ f=\rho+\frac{1-\rho}{2}f$,
so \(f=2\rho/(1+\rho)\le\eta/4\). A wrong guess uses fewer than two tests in expectation, while a correct guess is eventually accepted and uses fewer than \(8/\eta\) tests in expectation.

Once accepted, \(\widehat x\) can be produced using at most \(s\) XOR queries, because
\(
\widehat x+A=q+s_A\in S.
\)
Thus, XORing \(A\) with the appropriate stored basis equations yields \((\widehat x,\widehat x\cdot x)\).

\smallskip
\textit{Query count.}
We now count queries. The expected number of seeds up to the first correct one is at most \(\eta^{-1}\). Their creation and screening use \(O(\eta^{-1})\) generations, so including preprocessing,
\[
\mathbb E G=O(w+\eta^{-1})=O(w+2^{n-w}).
\]

The expected number of wrong seeds accepted for synthesis is at most
\(\eta^{-1}\cdot(\eta/4)=1/4\). Since the first correct seed is accepted with probability one, the expected number of synthesized guesses is at most \(5/4\).

\smallskip
\textit{Live pool.} It remains to ensure that the policy has no persistent state beyond its live pool. After preprocessing, retain
$Z:=\Xor((u_1,\beta_1),(u_1,\beta_1))=(0,0),$
and, whenever a seed is accepted, create \(Z':=\Xor(Z,Z)\). The ordered pool then determines the current phase. During preprocessing it contains the basis and at most two temporary vectors; during screening it contains the basis, \(Z\), the seed, and at most one temporary test equation; during synthesis it contains the basis, \(Z,Z'\), and one accumulator. During preprocessing, the first temporary vector in a two-vector extension is identifiable from the ordered pool because adjoining it makes the span degenerate, whereas the preceding basis spans a nondegenerate subspace. Since \(s\le w-3\), each pool has size at most \(w\). Moreover, all quantities used by the policy can be recomputed from this ordered pool, so no additional persistent memory is required.

The marker \(Z\) costs one recombination, and each synthesized guess costs at most \(s+1\) further recombinations. Hence
$\mathbb E R\le1+\frac54(s+1)=O(w).$

\smallskip 
\textit{Truncation.}
It now remains to show that the correct answer is produced in time. If \(\tau\) is the number of queries until the first correct guess is synthesized, then for a universal \(C_0>0\),
\[
\mathbb E\tau\le C_0(w+2^{n-w}).
\]
Set \(T:=\lceil3C_0(w+2^{n-w})\rceil\). By Markov's inequality, \(\Pr(\tau>T)\le1/3\), and therefore
\[
V_T^\pi(\mathfrak G_x^{\mathrm{par}})
\ge\Pr(\tau\le T)\ge 2/3.
\]
After increasing the universal constant if necessary, \(T\le C(w+2^{n-w})\).
\end{proof}

The proof of Theorem~\ref{thm:parity-tradeoff} follows from Lemma~\ref{lem:lower} and Lemma~\ref{lem:upper}.

\subsection{Proof of Theorem~\ref{thm:parity-memory}.}

\ParityMemoryBarrier*

\begin{proof}
For achievability, generate \(2n\) independent equations \((a_i,b_i)_{i=1}^{2n}\). The probability that \(a_1,\ldots,a_{2n}\) fail to span \(\mathbb F_2^n\) is at most \(2^{-n}\): if they do not span, there is some nonzero \(y\in\mathbb F_2^n\) such that \(a_i\cdot y=0\) for every \(i\). For each fixed \(y\ne0\), this occurs with probability \(2^{-2n}\), and hence, by a union bound,
\[
\Pr\left[\operatorname{span}(a_1,\ldots,a_{2n})\ne\mathbb F_2^n\right]
\le (2^n-1)2^{-2n}\le2^{-n}.
\]

On the complement event, retain \(n\) linearly independent equations \((u_1,\beta_1),\ldots,(u_n,\beta_n)\). They uniquely determine \(x\), which the policy computes by Gaussian elimination. Since \(u_1,\ldots,u_n\) form a basis, there are \(\lambda_1,\ldots,\lambda_n\in\mathbb F_2\) such that \(x=\bigoplus_{i:\lambda_i=1}u_i\). If \(x\ne0\), XORing the corresponding equations produces
\[
\left(\bigoplus_{i:\lambda_i=1}u_i,\bigoplus_{i:\lambda_i=1}\beta_i\right)
=(x,x\cdot x)
\]

using at most \(n-1\) recombinations. If \(x=0\), XORing any retained equation with itself produces \((0,0)=(x,x\cdot x)\). Thus the policy uses at most \(3n\) queries and succeeds with probability at least \(1-2^{-n}\ge2/3\). Its persistent state consists of \(O(n)\) vectors of \(n+1\) bits together with \(O(n^2)\) bits for the linear-algebra computation, and therefore uses at most \(Cn^2\) bits of memory for a universal constant \(C\).

For the lower bound, set \(c:=1/40\). By Theorem~1 of~\citep{raz2018fast}, there exists \(\beta>0\) such that any streaming algorithm using at most \(cn^2\) bits of memory and at most \(2^{\beta n}\) independent samples \((a,a\cdot x)\), with \(a\sim\operatorname{Unif}(\mathbb F_2^n)\), outputs \(x\) with probability at most \(O(2^{-\beta n})\). Set \(\alpha:=\beta/2\).

Fix a policy \(\pi\) using at most \(cn^2\) bits of memory and \(T\le2^{\alpha n}\). Let \(\pi^{(0)}\) be the coupled execution in which every reward bit is replaced by \(0\), using the same internal randomness and generation outcomes as \(\pi\), and write \(q_t^{(0)}=(a_t^{(0)},b_t^{(0)})\) for its \(t\)-th queried candidate. Since the reward differs from \(0\) only upon success, the two executions coincide up to and including the first successful query. Therefore
\[
V_T^\pi(\mathfrak G_x^{\mathrm{par}})
\le
\sum_{t=1}^T
\Pr\left[q_t^{(0)}=(x,x\cdot x)\right].
\]

Fix \(t\le T\). Construct a streaming parity learner \(A_t\) that simulates the first \(t\) rounds of \(\pi^{(0)}\) and outputs \(a_t^{(0)}\). Whenever \(\pi^{(0)}\) invokes generation, \(A_t\) reads the next fresh parity sample; recombination uses no new sample and is simulated from the persistent state. By the memory restriction, all information from previous rounds that can affect future actions, including every candidate used in a later recombination, is encoded in this state. Hence \(A_t\) uses at most \(cn^2\) bits of memory and reads at most \(t\le T\le2^{\alpha n}\le2^{\beta n}\) samples.

Fixing the internal randomness of \(\pi\) gives a deterministic streaming algorithm with the same bounds, so Raz's theorem yields \(\Pr[a_t^{(0)}=x]\le O(2^{-\beta n})\). Averaging over the internal randomness gives the same bound for randomized \(\pi\). Since every generated or recombined candidate is a valid parity equation, \(a_t^{(0)}=x\) if and only if \(q_t^{(0)}=(x,x\cdot x)\). Thus
\[
\Pr\left[q_t^{(0)}=(x,x\cdot x)\right]\le O(2^{-\beta n}).
\]

Summing over \(t\le T\) gives
\[
V_T^\pi(\mathfrak G_x^{\mathrm{par}})
\le T\cdot O(2^{-\beta n})
\le O(2^{-(\beta-\alpha)n})
=O(2^{-\alpha n}),
\]
as required.
\end{proof}

    \section{Proof of Theorem~\ref{thm:one-dimensional-gaussian}: A One-Dimensional Gaussian Phase transition}
\label{app:one-dimensional-gaussian}

Next, we give the proof of the one dimensional Gaussian Phase transition. 
\onedimgaussian*

\begin{proof}
Fix the framework from Section~\ref{sec:one-dimensional-gaussian}. Let \(c_G,c_M,c_R\) denote the costs of generation, mutation, and recombination, respectively, and let \(c_{\min}:=\min\{c_G,c_M,c_R\}\). For a policy \(\pi\) making \(T\) queries, let \(z_t\) be the point queried at time \(t\), and let \(M_t:=\max_{s\le t} z_s\). Since \(U(x)=x\), we have \(V_T^\pi(\mathfrak G)=\E M_T\). Next, we consider the three cases $\alpha <,=$, or $> 0$. 

\medskip \noindent
\subsection*{If $\alpha>0$}
We first prove the upper bound in the contractive case \(\mu<1\) and \(\lambda_1+\lambda_2<1\). Fix \(u>0\), and consider any policy of cost at most \(n\), making \(T\le n/c_{\min}\) queries. Decompose according to the first time the running maximum exceeds \(u\):
\begin{equation}\label{eq:prob-Mn-gtr-u}
\mathbb P(M_T>u)
=
\sum_{t=1}^T
\mathbb P(M_t>u,\ M_{t-1}\le u),
\end{equation}
where \(M_0=-\infty\). Condition on the the filtration $\mathcal F_{t-1}$ generated by all random variables up to time $t-1$. Respectively denote the event that the \(t\)-th query is a generation, mutation or recombination by $G,M,R$. Then, on a generation event, \(\mathbb P(z_t>u, \text{ G} \mid \mathcal F_{t-1})\le \exp(-u^2/2)\). If it is a mutation of some previously queried \(x\), then on the event \(M_{t-1}\le u\) we have \(x\le u\), and therefore
\[
\mathbb P(z_t>u, \text{ M} \mid \mathcal F_{t-1})
=
\mathbb P(\mu x+\sigma_M Z>u)
\le
\mathbb P\left(Z>\frac{1-\mu}{\sigma_M}u\right).
\]

If it is a recombination of \(x,y\), then on \(M_{t-1}\le u\) we have \(x,y\le u\), and since \(\lambda_1,\lambda_2\ge0\),
$$
\mathbb P(z_t>u, \text{ R} \mid \mathcal F_{t-1})
=
\mathbb P(\lambda_1x+\lambda_2y+\sigma_R Z>u)
\le
\mathbb P\left(Z>\frac{1-\lambda_1-\lambda_2}{\sigma_R}u\right).
$$

By the definition of \(\alpha\) and the Gaussian tail bound \(\mathbb P(Z>r)\le \exp(-r^2/2)\), each summand in \eqref{eq:prob-Mn-gtr-u} is at most \(\exp(-\alpha^2u^2/2)\). Hence
$$
\mathbb P(M_T>u)\le \frac{n}{c_{\min}}\exp(-\alpha^2u^2/2).
$$

Let \(u_n=(1+\varepsilon)\alpha^{-1}\sqrt{2\log n}\). Then
$$
\E M_T
\le
u_n+\int_{u_n}^{\infty}\mathbb P(M_T>t)\,dt
\le
u_n+\frac{n}{c_{\min}}\int_{u_n}^{\infty}\exp(-\alpha^2t^2/2)\,dt
=
(1+\varepsilon+o(1))\alpha^{-1}\sqrt{2\log n}.
$$

Since \(\varepsilon>0\) is arbitrary, \(V_n^*(\mathfrak G)\le(1+o(1))\alpha^{-1}\sqrt{2\log n}\).

\subsection*{Matching lower bound for $\alpha>0$} If \(\alpha=1\), the generation-only policy gives \(\lfloor n/c_G\rfloor\) independent \(N(0,1)\) samples, so the standard Gaussian maximum asymptotics give \(\E M_{\lfloor n/c_G\rfloor}=(1+o(1))\sqrt{2\log n}\).

It remains to consider the case where the minimum in \(\alpha\) is attained by mutation or recombination. We treat both cases at once. Let \(q<1\) and \(\sigma>0\) denote the coefficient and noise level of the best contractive operator: for mutation, \(q=\mu\) and \(\sigma=\sigma_M\); for recombination, \(q=\lambda_1+\lambda_2\) and \(\sigma=\sigma_R\), using the same previous candidate twice. Let \(c_*\) denote the cost of this operator. Thus \(\alpha=(1-q)/\sigma\).

Fix \(\varepsilon\in(0,1)\), and set \(u=(1-\varepsilon)\alpha^{-1}\sqrt{2\log n}\). Let \(L:=\lceil\log n\rceil\), \(m:=\lfloor(n-c_GL)/(c_*L)\rfloor\), and \(u_j:=ju/L\) for \(j=0,\ldots,L\). The policy first makes \(L\) generation queries. With probability at least \(1-2^{-L}\), their maximum is nonnegative; on this event set the current \textit{candidate} to be any generated point with value at least \(u_0=0\).

The policy then climbs the levels \(u_1,\ldots,u_L\). At level \(j\), assuming the current candidate has value at least \(u_{j-1}\), it applies the chosen operator \(m\) times to the current candidate. In the recombination case it uses the current candidate as both parents. If any of the \(m\) outputs has value at least \(u_j\), it sets the current candidate to the first such output and moves to level \(j+1\). Otherwise the policy fails.
Conditional on reaching level \(j-1\), each attempt at level \(j\) succeeds with probability at least
\[
\mathbb P(q u_{j-1}+\sigma Z\ge u_j)
=
\mathbb P\left(Z\ge \frac{u_j-q u_{j-1}}{\sigma}\right).
\]

For every \(j\le L\),
\[
\frac{u_j-q u_{j-1}}{\sigma}
=
\frac{(1-q)u_{j-1}+u/L}{\sigma}
\le
\alpha u+\frac{u}{\sigma L}
=
(1-\varepsilon)\sqrt{2\log n}+o(1)
\le
(1-\varepsilon/2)\sqrt{2\log n}
\]

for all sufficiently large \(n\). Therefore, uniformly over \(j\),
\[
\mathbb P(q u_{j-1}+\sigma Z\ge u_j)
\ge
\mathbb P\left(Z\ge (1-\varepsilon/2)\sqrt{2\log n}\right)
\ge
n^{-(1-\varepsilon/2)^2+o(1)}.
\]

Since \(m=n^{1+o(1)}\), the probability of failing at a fixed level is at most
\[
\exp\left(-n^{1-(1-\varepsilon/2)^2+o(1)}\right)=o(1/L).
\]

A union bound over the \(L\) levels shows that the policy reaches \(u_L=u\) with probability \(1-o(1)\). Also, the negative part of the maximum after the initial \(L\) generations is negligible: if \(G_L\) is the maximum of those generations, then
\[
\E[-G_L\mathbf 1\{G_L<0\}]
=
\int_0^\infty \mathbb P(G_L<-t)\,dt
=
\int_0^\infty \Phi(-t)^L\,dt
=o(\sqrt{\log n}).
\]

Thus, this policy satisfies \(\E \left[M_{L+mL} \right]\ge (1-o(1))u\). Since \(\varepsilon>0\) was arbitrary it holds that \(V_n^*(\mathfrak G)\ge(1-o(1))\alpha^{-1}\sqrt{2\log n}\). This completes the proof for the case $\alpha > 0$. 

\subsection*{If $\alpha=0$}
If \(\mu=1\), generate one initial candidate. Then repeatedly make two mutation queries from the current candidate and keep the larger of the two children. Each two-query block increases the current candidate in expectation by \(\sigma_M\E\max\{Z_1,Z_2\}=\sigma_M/\sqrt{\pi}\). Since each block has cost \(2c_M\), \(V_n^*(\mathfrak G)\ge \Omega(n)\). The same argument applies when \(\lambda_1+\lambda_2=1\), by recombining the current candidate with itself and using \(\sigma_R\) and \(c_R\) in place of \(\sigma_M\) and \(c_M\).

\subsection*{If $\alpha<0$}
If \(\mu>1\), generate two initial candidates and keep the larger one, whose expectation is positive. Then repeatedly mutate the current candidate. If \(x_t\) denotes the current candidate after \(t\) such mutations, then \(\E x_t=\mu^t\E x_0\). Since each mutation has cost \(c_M\), we may take \(t=\Theta(n)\), so \(V_n^*(\mathfrak G)\ge \exp(\Omega(n))\). The same argument applies when \(\lambda_1+\lambda_2>1\), again by recombining the current candidate with itself.

\subsection*{Upper bounds for $\alpha\le0$.}
It remains to prove the corresponding upper bounds. Let
$
B_t:=\max_{s\le t}|z_s|$,
$\rho:=\max\{1,\mu,\lambda_1+\lambda_2\}$, and let
$
\bar\sigma:=\max\{1,\sigma_M,\sigma_R\}.
$
For every query,
$
B_t\le \rho B_{t-1}+\bar\sigma|Z_t|,
$
where \(B_0=0\). Hence,
$$
\E B_T
\le
\bar\sigma \cdot \E \left[|Z| \right] \sum_{j=0}^{T-1}\rho^j.
$$
Since every policy of cost at most \(n\) has \(T\le n/c_{\min}\), if \(\alpha=0\), then \(\rho=1\) and \(V_n^*(\mathfrak G)\le O(n)\), while if \(\alpha<0\), then \(\rho>1\) and \(V_n^*(\mathfrak G)\le \exp(O(n))\). Together with the lower bounds, this gives \(V_n^*(\mathfrak G)=\Theta(n)\) when \(\alpha=0\) and \(V_n^*(\mathfrak G)=\exp(\Theta(n))\) when \(\alpha<0\). This proves the theorem.
\end{proof}

	\section{Proof of Theorem~\ref{thm:all-three-short}: All Three Operators are Necessary}
    \label{app:allOps}

    In this section, we prove Theorem~\ref{thm:all-three-short} showing all three operators are essential. 
	
	\threeOps*

\begin{proof}
Set $\gamma:=1/100$. The GMR framework $\mathfrak G=(X,U,\mathcal D,c,\mathsf{Obs})$ is defined as follows.
Throughout, policies are allowed full feedback: after each query $z_t$, they observe $(z_t,U(z_t))$.
Fix positive integers $G,R,M$ and $c_G,c_R,c_M$, and set $B:=10Gc_G+Rc_R+Mc_M$.
Set $X:=[0,1]^2$ and $\eta:=\frac{1}{2M}$.
Let $\Lambda:=2^{R+10}$.
Let $\mathcal D_0=\{D_{\mathrm{gen}}\}$, where
$D_{\mathrm{gen}}((0,0))=1-\frac{1}{2G}$
and
$D_{\mathrm{gen}}((1,0))=\frac{1}{2G}$.
Let $\mathcal D_1=\{M_\uparrow\}$, where
$
M_\uparrow(x,y):=\bigl(x,\min\{y+\eta,1\}\bigr),
$
and let $\mathcal D_2=\{R_{\mathrm{mid}}\}$, where
$$
R_{\mathrm{mid}}\bigl((x_1,y_1),(x_2,y_2)\bigr):=\Bigl(\frac{x_1+x_2}{2},0\Bigr).
$$

For $k\ge3$, set $\mathcal D_k=\varnothing$, let
$\mathcal D:=\mathcal D_0\cup\mathcal D_1\cup\mathcal D_2$, and assign costs
$c(D_{\mathrm{gen}}):=c_G$, $c(M_\uparrow):=c_M$, and $c(R_{\mathrm{mid}}):=c_R$.
Define $x_j:=2^{-j}$ and $a_j:=(x_j,0)$ for $j=1,\ldots,R$, and let $b:=(x_R,1/2)$. The hidden objective $U:X\to[0,3]$ is
\[
U(x,y):= 3\cdot 2^{-\Lambda |x-x_R|}\cdot \max\Bigl\{0,\,1-\Bigl|y-\frac12\Bigr|\Bigr\}.
\]

We first lower bound the benchmark $V_B^*(\mathfrak G)$. Consider the following policy $\pi^\star$. It first performs $10G$ generations. If both $(0,0)$ and $(1,0)$ have appeared, it keeps one copy of each and queries
$
a_1=R_{\mathrm{mid}}\bigl((0,0),(1,0)\bigr)=\Bigl(\frac12,0\Bigr).
$
For $j=2,\dots,R$, it recursively queries
$
a_j=R_{\mathrm{mid}}\bigl((0,0),a_{j-1}\bigr).
$

A simple induction shows that these points are exactly $a_j=(2^{-j},0)$, for $j=1,\ldots,R$. The policy then applies $M_\uparrow$ exactly $M$ times to $a_R$. Since each mutation increases the $y$-coordinate by $\eta=1/(2M)$, after $M$ mutations it reaches $b=(x_R,1/2)$, where $U(b)=3$. If either $(0,0)$ or $(1,0)$ is missing after the first $10G$ generations, the policy performs arbitrary $R$ recombination queries and $M$ mutation queries instead.

Thus, on every execution path, $\pi^\star$ performs exactly $10G$ generations, $R$ recombinations, and $M$ mutations, and therefore has cost exactly $B$. Its failure event is that at least one of $(0,0)$ and $(1,0)$ is not seen in the first $10G$ generations. By the union bound,
\[
\mathbb P\!\left[\{(0,0),(1,0)\}\nsubseteq\{z_1,\ldots,z_{10G}\}\right]
\le
\left(1-\frac{1}{2G}\right)^{10G}
+\left(\frac{1}{2G}\right)^{10G}
\le e^{-5}+2^{-10}.
\]
Therefore, $V_B^*(\mathfrak G) \ge 3(1-e^{-5}-2^{-10})$.

Now fix any policy $\pi$ and $T\in\mathbb N$. Let $G_T^\pi,R_T^\pi,M_T^\pi$ denote the numbers of generation, recombination, and mutation queries, respectively, among the first $T$ rounds, and let $W_T:=\max_{t\in[T]}U(z_t)$.
For the analysis, define the events
$
\mathcal E_G:=\{G_T^\pi<\gamma G\}$, 
$\mathcal E_R:=\{R_T^\pi<\gamma R\}$, and 
$\mathcal E_M:=\{M_T^\pi<\gamma M\}$; 
 let $\mathcal S:=\mathcal E_G^c\cap\mathcal E_R^c\cap\mathcal E_M^c$.

\subsection*{Too few generations.}
Let $\mathcal H$ be the event that $(1,0)$ is generated within the first $T$ rounds. On $\mathcal E_G$, fewer than $\gamma G$ generation queries are made, so
\[
\mathbb P(\mathcal E_G\cap\mathcal H) \le \frac{\lceil\gamma G\rceil-1}{2G} < \frac{\gamma}{2}.
\]

On $\mathcal E_G\cap\mathcal H^c$, every generated point has $x$-coordinate $0$. Mutation does not change the $x$-coordinate, and midpoint recombination of points with $x$-coordinate $0$ again has $x$-coordinate $0$. Therefore, $W_T\le 3\cdot 2^{-\Lambda x_R} =3\cdot 2^{-1024}$ on $\mathcal E_G\cap\mathcal H^c$.
Since always $W_T\le3$, it follows that
$$
\mathbb E\bigl[W_T\mathbf 1_{\mathcal E_G}\bigr] \le \frac{3\gamma}{2} +3\cdot 2^{-1024}\,\mathbb P(\mathcal E_G).
$$

\subsection*{Too few recombinations.}
On $\mathcal E_R$, let $L:=\lfloor\gamma R\rfloor$. Every queried point has $x$-coordinate equal to a dyadic rational whose denominator divides $2^L$. Indeed, generated points have $x$-coordinates $0$ and $1$, mutation does not change the $x$-coordinate, and each midpoint recombination can increase the dyadic denominator by at most one power of $2$. Since $L<R$, the target coordinate $x_R=2^{-R}$ cannot be reached, and every such dyadic $z$ satisfies $|z-x_R|\ge2^{-R}$. Hence $W_T\le 3\cdot 2^{-\Lambda 2^{-R}} =3\cdot 2^{-1024}$ on $\mathcal E_R$.

\subsection*{Too few mutations.}
On $\mathcal E_M$, every generation and every recombination produces a point with $y$-coordinate $0$, and fewer than $\gamma M$ mutations are made in total. Since each mutation increases the $y$-coordinate by exactly $1/(2M)$, every queried point has $y\le\gamma/2$. Therefore, $W_T \le \frac{3(1+\gamma)}{2}$ on $\mathcal E_M$.

It remains to combine the three bounds. The four events
$
\mathcal S,
\mathcal E_G,
\mathcal E_R\setminus\mathcal E_G,
\mathcal E_M\setminus(\mathcal E_G\cup\mathcal E_R)
$
form a partition. Hence, using $3\cdot2^{-1024}\le 3(1+\gamma)/2$, we obtain
$
V_T^\pi(\mathfrak G)=\mathbb E W_T.
$
By the above partition and the bounds from the three cases,
$$
V_T^\pi(\mathfrak G)
\le
3\,\mathbb P(\mathcal S)
+\frac{3\gamma}{2}
+3\cdot2^{-1024}\,\mathbb P(\mathcal E_G)
+3\cdot2^{-1024}\,\mathbb P(\mathcal E_R\setminus\mathcal E_G)
+\frac{3(1+\gamma)}{2}\,
\mathbb P\bigl(\mathcal E_M\setminus(\mathcal E_G\cup\mathcal E_R)\bigr).
$$

Since the three shortage parts have total probability $1-\mathbb P(\mathcal S)$, this further gives
\[
V_T^\pi(\mathfrak G)
\le
3\,\mathbb P(\mathcal S)
+\frac{3(1+\gamma)}{2}\bigl(1-\mathbb P(\mathcal S)\bigr)
+\frac{3\gamma}{2}.
\]

If $\mathbb P(\mathcal S)<1/2$, then, since $\gamma=1/100$,
\[
V_T^\pi(\mathfrak G) < \frac32+\frac{3(1+\gamma)}{4}+\frac{3\gamma}{2} =\frac{909}{400} < \frac9{10}\,3(1-e^{-5}-2^{-10}) \le \frac9{10}V_B^*(\mathfrak G).
\]

Therefore, every $\pi$ and $T$ satisfying $V_T^\pi(\mathfrak G)\ge\frac9{10}V_B^*(\mathfrak G)$ must satisfy $\mathbb P(\mathcal S)\ge1/2$. By the definition of $\mathcal S$, with probability at least $1/2$, the policy simultaneously applies at least $\gamma G$ generations, $\gamma R$ recombinations, and $\gamma M$ mutations within its first $T$ rounds. This proves the theorem.
\end{proof}

\section{Discussion and Conclusion}
\label{sec:discussion}

We introduced the GMR framework as a model for ML-guided genetic algorithms.
The motivation is that classical genetic-algorithm models do not capture the
setting of iterative LLM prompting: mutation and recombination are not blind random
perturbations and crossovers, but learned operators whose goal is to improve
the input solutions. This leads to a different query-allocation problem: how should a policy's budget be spent on generations, mutations, and recombinations? Unlike in classical GAs, each operation here may be expensive, so it becomes central to better allocate the query budget.

Our main message is that optimization-guided operators can behave very
differently from classical averaging or random-splicing operations, and that
generation, mutation, and recombination can play genuinely different roles in
theory. Theorem~\ref{thm:one-dimensional-gaussian} shows that stronger guided
operators can lead to much faster improvement than averaging-type operators, and gives a phase transition by the drift of the operator.
Theorem~\ref{thm:all-three-short} shows a simple construction where each operator is necessary. 
Finally, for parity learning, viewed as a GMR instance, we show that {\em diversity} is necessary: generating the hidden target on length $n$ vectors for policies of live pool size $w$ requires $\Theta(w+2^{n-w})$ queries. Moreover, a quadratic amount of memory overall is necessary: policies with $o(n^2)$ bits of memory fail with exponentially high probability for exponentially many queries.

\newpage

\bigskip

\bibliographystyle{plainnat}
	\bibliography{biblio}

\bigskip
\appendix

	\section{Notable Instantiation Examples of the GMR Framework}
    \label{sec:examples}
    
Many familiar stochastic search and model training procedures fit into our framework:

\subsection*{Multi arm bandits}
A fixed \(A\)-armed bandit instance can be modeled using only \(\mathcal D_0\).
Take \(X=[A]\times\mathbb R\), where a point \(z=(i,R)\) records that arm \(i\)
was pulled and reward \(R\) was observed, and define \(U(i,R)=R\).
\emph{Generation:} for each arm \(i\in[A]\), include one generation rule
\(D_i\in\mathcal D_0\). This rule samples \(z=(i,R)\), where \(R\sim\nu_i\) and
\(\nu_i\) is the fixed reward law of arm \(i\). Thus \(K=\{0\}\) and
\(\mathcal D_0=\{D_i:i\in[A]\}\). The feedback can be taken to be full feedback,
\(o_t=(z_t,U(z_t))\). Choosing \(D_i\) is exactly pulling arm \(i\), and the
reward randomness is the randomness of the generation rule.

\subsection*{Stochastic gradient and first-order methods}
These can be modeled using only \(\mathcal D_0\) and \(\mathcal D_1\), for a
fixed training problem and a fixed stochastic update procedure. For example,
let \(X\) be the space of optimization states \(x=(\theta,h)\), where
\(\theta\in\mathbb R^d\) is the current parameter vector and \(h\) stores any
internal memory, such as momentum variables or Adam statistics.
\emph{Generation:} an initialization rule \(D_{\mathrm{init}}\in\mathcal D_0\)
samples the initial state \(x_0=(\theta_0,h_0)\).
\emph{Mutation:} an update rule \(D_{\mathrm{upd}}\in\mathcal D_1\) maps a
current state \(x=(\theta,h)\) to the distribution of the next state obtained
by one stochastic update: it samples a minibatch, computes the corresponding
gradient or finite-difference estimate for the fixed training problem, and
returns the updated state \(x'=(\theta',h')\). Thus \(K=\{0,1\}\),
\(\mathcal D_0=\{D_{\mathrm{init}}\}\), and
\(\mathcal D_1=\{D_{\mathrm{upd}}\}\). The objective can be taken to be
\(U(\theta,h)\), the performance of the current iterate, e.g., minus validation
loss. The feedback can be taken to be full feedback, \(o_t=(z_t,U(z_t))\), or a
more limited feedback model depending on the algorithm being represented. The
training problem and the stochastic update rule are fixed parts of the
framework \(\mathfrak G\); they are not hidden varying objectives.

\subsection*{Genetic algorithms and evolution strategies}
Take \(X\) to be a space of individuals, and let \(U:X\to\mathbb R\) be the
fitness or performance objective.
\emph{Generation:} initial-population distributions \(\mathcal D_0\).
\emph{Mutation:} standard perturbation operators \(\mathcal D_1\).
\emph{Recombination:} crossover or more general multi-parent combination
operators \(\mathcal D_k\) for \(k\ge 2\). The feedback map \(\mathsf{Obs}\) can
represent the information exposed to the policy, such as full candidate and
fitness feedback or only objective-value feedback.
\end{document}